\documentclass{siamltex}

\usepackage[text={5.125in,8.25in},centering]{geometry}
\usepackage{graphicx}
\usepackage{epsfig} 
\usepackage{mathptmx} 
\usepackage{times} 
\usepackage{amsmath} 
\usepackage{amssymb}  
\usepackage{dsfont}
\usepackage{verbatim}

\usepackage{color}
\usepackage[dvipsnames]{xcolor}

\definecolor{amethyst}{rgb}{1, 0, 1}
\definecolor{blue-violet}{rgb}{0.54, 0.17, 0.89}
\definecolor{brightturquoise}{rgb}{0.03, 0.91, 0.87}

\usepackage{caption}
\usepackage{subcaption}

\usepackage{pgfplots}\pgfplotsset{compat=1.4}

\usepackage{pgfplots}\pgfplotsset{compat=1.4}
\definecolor{softgray}{rgb}{0.92,0.92,0.95}
\definecolor{softblue}{rgb}{0.90,0.92,1.00}
\definecolor{lightgray}{rgb}{0.12,0.12,0.55}

\definecolor{theframe}{gray}{0.75}
\definecolor{theblue} {rgb}{0.02,0.04,0.48}
\definecolor{thegrey} {gray}{0.5}
\definecolor{theshade}{gray}{0.98}
\definecolor{thered}  {rgb}{0.00,0.00,0.00}
\definecolor{thegreen}{rgb}{0.3,0.3,0.3}

\usepackage{color}
\definecolor{softblue}{rgb}{0.90,0.92,1.00}
\usepackage{mdframed}
\usepackage{makecell}

\usepackage{graphicx}
\usepackage{hyperref}

\numberwithin{equation}{section}

\newtheorem{lemmer}{Lemma}[section]
\newtheorem{remark}[lemmer]{Remark}
\newtheorem{counterexample}[lemmer]{Counterexample}

\renewcommand{\P}{{\mathbb P}}
\newcommand{\E}{{\mathbb E}}
\newcommand{\1}{{\mathds{1}}}

\newcommand{\Var}{\mbox{Var}}

\title{Kernel absolute summability is \\ only sufficient for RKHS stability}
 \author{%
    Mauro Bisiacco\thanks{Department of Information Engineering, University of Padova, Padova, Italy (bisiacco@dei.unipd.it)}
    \and Gianluigi Pillonetto\thanks{Department of Information Engineering, University of Padova, Padova, Italy (giapi@dei.unipd.it)}
    \hfill\today}
\begin{document}

\maketitle

\begin{abstract}
Regularized approaches have been 
successfully applied to linear system identification in recent years.
Many of them model unknown impulse responses exploiting 
the so called Reproducing Kernel Hilbert spaces (RKHSs) that
enjoy the notable property of being in one-to-one correspondence
with the class of positive semidefinite kernels.
The necessary and sufficient condition for a RKHS to be stable, i.e. to contain only BIBO stable linear dynamic
systems, has been known in the literature at least since 2006.
However, an open question still persists and concerns the equivalence of such condition 
with the absolute summability of the kernel. This paper provides a definite answer to this 
matter by proving that such correspondence does not hold.
A counterexample is introduced that illustrates the existence of stable RKHSs that are 
induced by non-absolutely summable kernels.
\end{abstract}

\section{Introduction}
\normalsize

The classical approach to linear system identification
uses parametric models of different orders.
In particular, a set of candidate structures 
that increase in complexity are selected. 
They are then typically fit to data using
Prediction Error Methods \cite{Ljung,Soderstrom} 
and the `best' model is selected using complexity measure such as Akaike information criterion 
or cross validation techniques \cite{Akaike1974,schwarz1978estimating,Hastie01}.
In the last years, alternative regularized approaches 
have attracted much attention within the control community.
They search for the unknown impulse response in flexible spaces 
that incorporate fundamental dynamic features, like stability, with complexity regulated by some continuous variables. 
In particular, infinite-dimensional spaces known as 
Reproducing Kernel Hilbert Spaces (RKHSs) are widely adopted.\\
RKHSs enjoy many important properties. They are in  
one to one correspondence with the class of positive semidefinite kernels $K$ and have
also fundamental connections with Gaussian processes when $K$ is seen as a covariance \cite{Kimeldorf71Bayes,Lukic,AravkinNN}.
RKHSs were introduced
to the machine learning community in \cite{Girosi:1997}. 
They permit to treat in a unified framework many different algorithms: 
the so called kernel-based methods 
\cite{PoggioMIT,Scholkopf01b} include smoothing splines \cite{Wahba1990}, regularization networks
\cite{Poggio90}, Gaussian regression \cite{Rasmussen} and support vector machines \cite{Drucker97,Vapnik98}.
But while in machine learning kernels are typically used to encode information on function smoothness,
control community's interest has been instead recently addressed to the building of 
RKHSs that include dynamic systems features. 
For instance, the so called stable spline kernel was introduced in \cite{SS2010} to 
model impulse responses that are smooth and decay exponentially to zero. 
It belongs to the much more general class of (BIBO) stable kernels that 
induce RKHSs containing only absolutely summable impulse responses.  
One fundamental question discussed in \cite{SurveyKBsysid}[Part III] was the necessary and sufficient condition 
for kernel stability. 
Nowadays, it is known that $K$ is stable if and only if it induces an integral operator
that maps the whole space of 
essentially bounded functions into the space of absolutely summable functions.
In \cite{Carmeli,DinuzzoSIAM15}, immediately after reporting such result and 
looking for a (in some sense) simpler stability test, 
authors mentioned kernel absolute summability as a sufficient condition.
The necessity was however left as an open problem. And ever since then, 
many papers have cited and used kernel summability as a stability check, 
without providing an answer to this question, e.g. see \cite{Darwish2015,ChenKS2015,Fujimoto2017,ChenKS2018}. 
This paper will face and solve such open question by showing that the equivalence does not hold.
Indeed, it will be proved that there exist stable RKHSs induced by non-absolutely summable
kernels. Our result thus provides a further important step towards a better understanding of RKHSs structures 
useful e.g. for system identification.\\ 
The paper is organized as follows. In Section \ref{Sec2} the problem statement is reported.
Section \ref{Sec3} describes a class of matrices that will be key to solve our problem.
In Section \ref{Sec4}, they are used to prove that absolute kernel summability
is not necessary for the existence of a linear integral operator 
from $\ell_{\infty}$ into $\ell_1$ not subject to  positive semidefinite constraints.
Section \ref{Sec5} then brings   
such constraints into the picture.  
Some properties of symmetric positive semidefinite matrices 
are first given. Next, they are used to reduce our central question 
to a particular problem in finite-dimensional spaces.
Section \ref{Sec6} reports a class of important positive semidefinite matrices
that build upon the matrices illustrated in Section \ref{Sec3}.
They are finally exploited in Section \ref{Sec7}
to prove that kernel absolute summability
is only sufficient for RKHS stability.

\section{Problem statement}\label{Sec2}


A RKHS is a special Hilbert space of functions where 
all the pointwise evaluators are continuous (bounded) linear
functionals. This property also implies that an RKHS
is in one to one correspondence with a 
symmetric and positive semidefinite kernel $K$, i.e. such that 
for any finite natural number $m$,
scalars $c_1,\ldots,c_m$ and elements $x_1,\ldots,x_m$ of the function domain, it holds that 
\begin{equation}\label{Kconstraint}
\sum_{i=1}^{m}\sum_{j=1}^{m}c_ic_j K(x_i,x_j) \geq 0, \quad K(x_i,x_j)=K(x_j,x_i). 
\end{equation}
One can prove that any element of an RKHS is the (possibly infinite) sum
of kernel sections, i.e. of functions of the type $K(x,\cdot)$. This property also suggests 
that vectors inherit the properties of $K$, e.g. continuous kernels define RKHSs of
continuous functions.\\
According to \cite{DinuzzoSIAM15}, kernels are said to be stable if they
induce stable RKHSs, i.e. containing 
only absolutely summable (causal) functions. 
Hence, the elements of such spaces can be 
interpreted as impulse responses of BIBO stable linear and time-invariant dynamic systems.
Without loss of generality, the discrete-time case will be considered.
The function domain is equal to the set of natural numbers $\mathbb{N}$ and
the RKHSs are made up of sequences. So, it is useful to introduce 
the  spaces $\ell_\infty$ and $\ell_1$ of
bounded and absolutely summable sequences of real numbers, respectively, i.e.
$$
\ell_\infty = \Big\{ \{u_i\}_{i \in {\mathbb{N}}}  \ \mbox{s.t.} \   \| u \|_\infty <  \infty \Big\},
$$
and
$$
\ell_1 = \Big\{ \{u_i\}_{i \in {\mathbb{N}}}   \ \mbox{s.t.} \   \| u \|_1 <  \infty \Big\},
$$
where
$$
\| u \|_\infty = \sup_{i \in {\mathbb{N}}}  |u_i| \quad \mbox{and} \quad \| u \|_1 = \sum_{i \in {\mathbb{N}}}  |u_i|. 
$$
Furthermore, it is also useful to see the kernel as an infinite-dimensional matrix
with the $(i,j)$-entries denoted 
by $K_{ij}$. 
Then, the following result states the necessary and sufficient condition
for $K$ to be stable.\\

\begin{theorem}[RKHS stability \cite{Carmeli}]\label{NecSuffStabKer}
Let $\mathcal{H}$ be the RKHS induced by $K: \mathbb{N} \times \mathbb{N} \rightarrow \mathbb{R}$. One then has 
\begin{equation} \label{CondNS}
 \mathcal{H} \subset \ell_1 \  \iff   \  \sum_{i=1}^\infty \left|  \sum_{j=1}^\infty u_j K_{ij} \right| < \infty  \ \ \forall u \in
\ell_{\infty}.
\end{equation}
\end{theorem}

\noindent This theorem, not surprisingly, shows that 
$\ell_{\infty}$ contains the key test functions to assess RKHS stability. 
But is it possible to find an alternative (and in some sense simpler) equivalent 
condition on $K$?
Following the discussions in \cite{Carmeli,DinuzzoSIAM15} subsequent to Theorem \ref{NecSuffStabKer},
kernel absolute summability, i.e. the property $\sum_{i,j} |K_{ij}|< \infty$, is an interesting candidate.
In fact, it is immediate to see that such condition is sufficient for stability but it is not yet known
if the equivalence with (\ref{CondNS}) holds. 
Hence, our problem is to understand if kernel summability is not only sufficient but also
necessary for a RKHS to be stable. 

\medskip

\begin{remark}
Theorem \ref{NecSuffStabKer} can be also described as follows. 
The kernel $K$ defines an acausal linear time-varying system: given an input (sequence) $u$,  
the output at instant $i$ is $\sum_{j=1}^\infty K_{ij} u_j$. Then, the kernel is stable  
if and only if such system maps every bounded input into a summable output.
RKHS stability thus involves (integral) linear operators from $\ell_{\infty}$ to $\ell_1$ and,
interestingly, we have not found any result on this kind of maps relevant for our analysis.
The reason is that the (control) literature has studied BIBO stability
considering linear transformations each representing 
a single dynamic system (and not a class of systems as done by a kernel). 
This has then produced 
conditions for an integral operator
to map $\ell_{\infty}$ into $\ell_{\infty}$, e.g. see \cite{Willems1970,Desoer1975}.
Our analysis is instead more difficult: to characterize
stable RKHSs it is necessary to consider a 
subclass of these operators with the range restricted to $\ell_1$ and
subject to the constraints (\ref{Kconstraint}).
\end{remark}

\section{A class of important matrices} \label{Sec3}
In this section we introduce and analyze a class of special matrices which will play a fundamental role 
to solve our problem.  First, it is useful to set up some additional notation. 
All the vectors are column vectors and, given $v$,  
$v_i$ represents its $i$-th entry. 
We use $p$ to indicate an integer ($p \ge 1$) that defines also 
the odd number  $m=2p+1$ and the corresponding power of two $n=2^m$. 
For any integer $r \ge 1$, we also introduce the following set
\begin{equation}\label{mathU}
{\mathcal U}_r := \{ \ v \in {\mathbb R}^r: v_i=\pm 1, \forall \ i=1,\dots,r \ \}.  
\end{equation}
Now, consider all the distinct vectors $v^{(i)} \in {\mathcal U}_m$ ($i=1,2,\dots,n$) consisting of exactly $m$ elements $\pm 1$ 
(ordering of the $v^{(i)}$ is irrelevant). Then, for any $n=2^3, 2^5, 2^7, \dots$, 
the special matrix $V^{(n)}$ of size $n \times m$ is given by
\begin{equation}\label{Vn}
V^{(n)}=\left[\begin{matrix} v^{(1)} & v^{(2)} & \dots & v^{(n)} \end{matrix} \right]^{\top}.
\end{equation}
For instance, if $p=1$ then $m=3$, $n=2^3=8$ and
\begin{equation}\label{VnExample} \small
V^{(8)}=\left(\begin{array}{ccc}1 & 1 & 1 \\1 & 1 & -1 \\1 & -1 & 1 \\1 & -1 & -1 \\-1 & 1 & 1 \\-1 & 1 & -1 \\-1 & -1 & 1 \\-1 & -1 & -1\end{array}\right)
\end{equation}
that shows how the rows of such matrices contain all the possible permutations of $\pm 1$.
We now introduce two norms for $V^{(n)}$.
The first one is 
\begin{equation}\label{1Norm}
\|V^{(n)}\|_1 := \sum_{i,j} \ |V^{(n)}_{ij}|
\end{equation}
where $V_{ij}^{(n)}$ are the entries (of values $\pm 1$) of $V^{(n)}$.
So,  $\|V^{(n)}\|_1$ is the $\ell_1$ norm 
understood as sum of the modules of all its entries. 
One thus has $\|V^{(n)}\|_1=mn$. 
The second alternative norm is 
\begin{equation}\label{OpNorm}
\|V^{(n)}\|_{\infty,1}:=\max_{\|u\|_{\infty}=1} \ \|V^{(n)}u\|_1.
\end{equation}
Note that (\ref{OpNorm}) is the norm of the linear operator $V^{(n)}:{\mathbb R}^m \rightarrow {\mathbb R}^n$ 
once ${\mathbb R}^m$ and ${\mathbb R}^n$ are equipped with the $\ell_{\infty}$ and the $\ell_1$ norms, respectively.
\medskip

\begin{lemma} \label{Lemma1} In evaluating $\|V^{(n)}\|_{\infty,1}$ in (\ref{OpNorm}), we can limit ourselves to consider the only vectors $u$ in ${\mathcal U}_m$ defined in (\ref{mathU}), i.e.
$$
\max_{\|u\|_{\infty}=1} \ \|V^{(n)}u\|_1 = \max_{u \in {\mathcal U}_m} \ \|V^{(n)}u\|_1.
$$
\end{lemma}
\begin{proof}
The proof exploits convexity and is reported just for sake of completeness. Letting $u=\left[\begin{matrix} a_1 & \dots & a_m \end{matrix} \right]^\top$, one has
$$
\|V^{(n)}u\|_1=\Sigma_{k=1}^n \ \left|\Sigma_{h=1}^m \ V_{kh}^{(n)}a_h\right|:=f(a_1,\dots,a_m).
$$
The function $f$ is convex over $\mathbb{R}^m$ (being the composition of convex maps given by absolute values and linear maps).
Fix any vector $a$ whose $i$-th entry satisfies $|a_i|<1$. 
Replace such entry with $1$, obtaining the vector $b$, or with $-1$, leading to $c$.
Convexity of $f$ thus ensures that \cite{RTR}
$$
f(a) \leq f(b) \ \ \mbox{or} \ \   f(a) \leq f(c).
$$
So, given any maximizer of the $f$ restricted over the compact ${\|u\|_{\infty}=1}$,
each of its entry (of modulus less than one) can be replaced with either  $1$ or $-1$ maintaining the optimality. 
\end{proof}
\medskip

\begin{lemma}  \label{Lemma2}  The value of $\|V^{(n)}u\|_1$, with $u \in {\mathcal U}_m$, is independent of the chosen $u$.
\end{lemma} 
\begin{proof} In evaluating $V^{(n)}u$, one can easily see that replacing the $i$-th entry of $u$ with its opposite
is equivalent to changing the sign of the $i-$th column of $V^{(n)}$. However, by the properties of $V^{(n)}$, changing the sign of a column 
corresponds to reordering the rows of $V^{(n)}$ since 
$\left[\begin{matrix} w_1^\top & \pm 1 & w_2^\top\end{matrix}\right]$ 
both belong to the list of $V^{(n)}$'s rows. In other words, changing the sign of any entry of $u$
just corresponds to a change of the sign of all the corresponding $V^{(n)}$'s columns that is 
equivalent to a suitable reordering of its rows. 
It is now clear that any $u \in {\mathcal U}_m$ leads to the same vector $V^{(n)}u$, apart from an entries reordering,
so that $\|V^{(n)}u\|_1$ does not depend on $u$.
\end{proof} 
\medskip

\begin{lemma} \label{Lemma3} The following relation holds true
$$
\|V^{(n)}\|_{\infty,1}=\|V^{(n)}u\|_1=2\Sigma_{h=0}^p \ \left(\begin{matrix} m \cr h \end{matrix} \right) \ (m-2h) 
$$
where $u$ is any vector in $\mathcal{U}_m$.
\end{lemma} 
\begin{proof} By Lemmas \ref{Lemma1} and \ref{Lemma2}, we easily have 
$$
\|V^{(n)}\|_{\infty,1}=\|V^{(n)}u\|_1 \quad \forall u \in {\mathcal U}_m.
$$
Thus, we can choose $u=v^{(1)}:=\left[\begin{matrix} 1 & \dots & 1\end{matrix}\right]^\top$ and evaluate $\|V^{(n)}v^{(1)}\|_1=\|w\|_1$, where $w_i=\Sigma_{j=1}^m \ V_{ij}^{(n)}=\Sigma_{j=1}^m \ v_j^{(i)}, \ i=1,2,\dots,n$, with $V_{ij}^{(n)}$ and $v_{j}^{(i)}$  denoting the entries of $V^{(n)}$ and $v^{(i)}$, respectively. 
The number of vectors  $v^{(i)}$ containing $h$ negative signs and $m-h$ positive signs is $\left(\begin{matrix} m \cr h \end{matrix}\right)$.
In addition, for such kind of vectors one has $\Sigma_{j=1}^m \ v_j^{(i)}=|m-2h|$.
Then, we easily obtain 
$$
\|V^{(n)}\|_{\infty,1}=\|V^{(n)}v^{(1)}\|_1=\Sigma_{h=0}^{m} \ \left(\begin{matrix} m \cr h \end{matrix}\right) \ |m-2h|=2\Sigma_{h=0}^p \ \left(\begin{matrix} m \cr h\end{matrix}\right) \ (m-2h)
$$
where the last equality derives from the symmetry of the two cases $h \le p$ and $h > p$. This concludes the proof. 
\end{proof} 



\section{The first counterexample}\label{Sec4}

In this Section we will obtain a first result about operators from 
$\ell_{\infty}$ into $\ell_1$ induced by infinite-dimensional matrices that are not subject to 
the positive semidefinite constraints (\ref{Kconstraint}). 
As it will be clear in the final part of the paper, this intermediate step
will be crucial for solving the question regarding RKHS stability.\\
Thanks to the results obtained in the previous section, 
the following equalities regarding two norms
are now available: 
\begin{eqnarray}\label{Norm1}
\|V^{(n)}\|_1&=& nm\\ \label{Norm2}
\|V^{(n)}\|_{\infty,1} &=& 2\Sigma_{h=0}^p \ \left(\begin{matrix}m \cr h \end{matrix} \right) \ (m-2h)
\end{eqnarray}
However, the expression of $\|V^{(n)}\|_{\infty,1}$ is not so appealing: 
evaluation is not available in closed form and appears somewhat complicated. 
Actually, the important point is the comparison between (\ref{Norm1}) and (\ref{Norm2}) for large $p$. 
For this reason, the next lemma defines the behaviour of $\|V^{(n)}\|_{\infty,1}$ as $p$ tends to $\infty$. 
It relies on a classical result of Probability Theory, the Central Limit Theorem.
\medskip

\begin{lemma}\label{Lemma4}
One has 
$$
\Sigma_{h=0}^p \ \left(\begin{matrix}m \cr h\end{matrix} \right)(m-2h) \simeq n\sqrt{\frac{p}{\pi}} \ \quad \mbox{as} \quad p \rightarrow +\infty.
$$
\end{lemma}

\begin{proof}  Let $x \sim {\mathcal B}(m)$ be a binomial random variable 
assuming value 0 or 1 with equal probability, i.e. 
$$
\P(x=h)=\frac{1}{n}\left(\begin{matrix} m \cr h \end{matrix} \right), \ h=0,1,2,\dots,m.
$$
Its mean and variance are so given by
$$
\E(x)=\frac{m}{2}, \quad \Var(x)=\frac{m}{4}.
$$
By defining $f(x)=m-2x$ for $x \le p$ and $f(x)=0$ elsewhere, one has
$$
\E\big(f(x)\big)=\Sigma_{h=0}^p \ \P(x=h)(m-2x)=\frac{1}{n} \Sigma_{h=0}^p \ \left(\begin{matrix}m \cr h \end{matrix}  \right)(m-2h).
$$
For $p$ and, consequently, $m=2p+1$ as well as $n=2^m$ large enough,
the evaluation of $\E\big(f(x)\big)$ can be obtained through the normal approximation.
Letting $\Phi(a;b,c)$ be the Gaussian distribution evaluated at $a$ with mean $b$ and variance $c$, one has
$$
\P(x \leq a) \simeq \Phi\Big(a;\frac{m}{2},\frac{m}{4}\Big).
$$
Hence, we obtain
\begin{eqnarray*}
\E\big(f(x)\big) &\simeq& \int_0^{\frac{m}{2}} \ \sqrt{\frac{2}{\pi m}}e^{-\frac{\left(x-\frac{m}{2}\right)^2}{\frac{m}{2}}}(m-2x)dx \\
&=&-2\int_0^{\frac{m}{2}} \ \sqrt{\frac{2}{\pi m}}e^{-\frac{\left(x-\frac{m}{2}\right)^2}{\frac{m}{2}}}\left(x-\frac{m}{2}\right)dx \\
&=&-2\int_{-\frac{m}{2}}^0 \ \sqrt{\frac{2}{\pi m}}e^{-\frac{y^2}{\frac{m}{2}}}ydy =  \sqrt{\frac{m}{2 \pi}}\int_{-\frac{m}{2}}^0 \ \left[-e^{-\frac{2y^2}{m}}\right]d\left(\frac{2y^2}{m}\right) \\
&=& \sqrt{\frac{m}{2\pi}}\left[1-e^{-\frac{m}{2}}\right] \simeq \sqrt{\frac{p}{\pi}}.
\end{eqnarray*}
For large $p,m,n$ this indeed implies 
$$
\Sigma_{h=0}^p \ \left(\begin{matrix}m \cr h\end{matrix} \right)(m-2h) \simeq n\sqrt{\frac{p}{\pi}}
$$
and completes the proof.
\end{proof}
\medskip

Now, recall from the discussion in Section \ref{Sec2} that we are interested
in linear operators from $\ell_{\infty}$ into $\ell_1$ defined by means of an infinite matrix, i.e.
\begin{equation}\label{OpV}
{\mathcal V}: \ \ell_{\infty} \ \rightarrow \ \ell_1, \quad {\mathcal V}(u)=y
\end{equation}
where
\begin{equation*}
y_i=\sum_{h=1}^{+\infty} \ V_{ih}u_h, \ i=1,2,\dots
\end{equation*}
Defining the $\ell_1$ norm of ${\mathcal V}$ as 
\begin{equation}\label{Norm1Op}
\|{\mathcal V}\|_1:=\Sigma_{i,j=1}^{+\infty} \ |V_{ij}|
\end{equation}
we say that ${\mathcal V}$ is absolutely summable if and only if $\|{\mathcal V}\|_1<+\infty$. 
The $\|{\mathcal V}\|_1$ is different from the norm of the operator, defined by 
\begin{equation}\label{Norm2Op}
\|{\mathcal V}\|:=\max_{\|u\|_{\infty}=1} \ \|{\mathcal V}u\|_1
\end{equation}
($\sup$ is usually used instead of $\max$, but as clear in what follows no distinction is needed).
Then, the linear operator ${\mathcal V}$ is bounded (continuous) if and only if $\|{\mathcal V}\|<+\infty$. 
While absolute summability implies boundedness, the converse is false as the next explicit counterexample 
(that represents the first main result of this paper) will show. 
\medskip

\begin{counterexample}  A linear operator ${\mathcal V}: \ \ell_{\infty} \ \rightarrow \ \ell_1$ can be bounded 
even if it is not absolutely summable. 
\end{counterexample}

\begin{proof} Consider the following version of $V^{(n)}$ suitably scaled in such a way that its $\ell_1$ norm becomes $\frac{1}{p}$:
$$
V_*^{(n)}:=\frac{1}{pmn}V^{(n)}=\frac{1}{p\|V^{(n)}\|_1}V^{(n)}.
$$
One thus also has
$$
\|V_*^{(n)}\|_1=\frac{1}{pmn}\|V^{(n)}\|_1=\frac{1}{p}.
$$

Recalling (\ref{Norm1}), (\ref{Norm2}) and using Lemma \ref{Lemma4}, it follows that
\begin{equation}\label{A}
\frac{\|V^{(n)}\|_{\infty,1}}{\|V^{(n)}\|_1} \simeq \frac{2n\sqrt{\frac{p}{\pi}}}{nm} \simeq \frac{1}{\sqrt{\pi p}}
\end{equation}
for $p,m,n$ large enough.
From such equation, one also easily obtains that
$$
\|V_*^{(n)}\|_{\infty,1} \simeq \frac{1}{p\sqrt{\pi p}}
$$
still for $p, m, n$ large enough. 
Now, let us define the following infinite matrix 
$$
V=\diag(V_*^{(n(1))}, V_*^{(n(2))}, \dots),
$$
where $n( p ):=2^{2p+1}$. The block diagonal structure allows to partition $u \in \ell_{\infty}$ as 
$$
u=\left[\begin{matrix}u_3^\top & u_5^\top & \dots\end{matrix} \right]^\top
$$ and 
similarly $y \in \ell_1$ as 
$$
y=\left[\begin{matrix}y_{n(1)}^\top & y_{n(2)}^\top & \dots\end{matrix} \right]^\top.
$$ 
Moreover, any finite subvector $y_{n(p)}$ 
only depends on $u_{2p+1}$ by means of the matrix $V_*^{(n(p))}$. We then obtain that the 
linear operator ${\mathcal V}$ associated with the infinite matrix $V$ satisfies
$$
\|{\mathcal V}\|_1=\sum_{p=1}^{+\infty} \ \frac{1}{p} = +\infty
$$
and
$$
\ \|{\mathcal V}\|=\sum_{p=1}^{+\infty} \ \|V_*^{(n(p))}\|_{\infty,1}<+\infty,
$$
as a simple consequence of the convergence of the series $\sum_{p=1}^{+\infty} \ \frac{1}{p\sqrt{p}}$. 
Therefore, we have indeed found an operator ${\mathcal V}$ associated with the infinite matrix $V$ 
that is bounded even if $V$ is not absolutely summable. 
\end{proof}
\medskip
\begin{remark} The comment under (\ref{Norm2Op}) on the use of $\max$ in place of $\sup$ finds now the following explanation.
For all the inputs $u$ in the set ${\mathcal U}_{\infty}:=\{ \ u \in \ell_{\infty}: \ u_i=\pm 1, \ \forall \ i=1,2,\dots \ \}$, 
that corresponds to the infinite-dimensional version of (\ref{mathU}), it holds that 
$$
\|y\|_1=\|{\mathcal V}u\|_1=\|{\mathcal V}\|=\|{\mathcal V}\| \cdot \|u\|_{\infty}.
$$
\end{remark}

\section{Some properties of symmetric positive semidefinite matrices and problem reduction to finite-dimensional spaces}
\label{Sec5}

In the previous part we have provided some new insights on the maps 
from $\ell_{\infty}$ into $\ell_1$ without considering the constraints
(\ref{Kconstraint}). Now, we want to address the 
symmetric and positive semidefinite case.\\ 
In what follows, $M^{(k)}$ indicates a matrix of size $k \times k$
satisfying $M^{(k)}=M^{(k)\top} \ge 0$. Thus, it belongs to 
the set of $k \times k$ symmetric and positive semidefinite matrices that we denote by
${\mathcal C}_k$.  
As before, we are interested in obtaining relationships between the two norms $\|M^{(k)}\|_1$ and $\|M^{(k)}\|_{\infty,1}$
defined exactly as in  (\ref{1Norm}) and (\ref{OpNorm}). The sequence of real numbers $\lambda(k)$
introduced in the next lemma provides a fundamental connection.\\ 

\begin{lemma} \label{Lemma5}  For any $k \ge 1$, 
$$
\lambda(k):=\min_{M^{(k)} \in {\mathcal C}_k: \ \|M^{(k)}\|_1=1} \ \|M^{(k)}\|_{\infty,1}
$$
is well-defined and satisfies the following properties:\\

\begin{itemize}
\item $0 \le \lambda(k) \le 1$;
\item $\|M^{(k)}\|_1 \ge \|M^{(k)}\|_{\infty,1} \ge \lambda(k)\|M^{(k)}\|_1$, for any $M^{(k)} \in {\mathcal C}_k$;
\item there exist at least two matrices $M_1^{(k)}, M_2^{(k)} \in {\mathcal C}_k$ such that 
$$
\|M_1^{(k)}\|_{\infty,1}=\|M_1^{(k)}\|_1
$$ 
and 
$$
\|M_2^{(k)}\|_{\infty,1}=\lambda(k)\|M_2^{(k)}\|_1
$$ 
(that is equivalent to saying that better bounds cannot be found);
\item the sequence $\lambda(k)$ is monotone non-increasing.
\end{itemize}

\end{lemma}

\begin{proof} Assume $u \in {\mathbb R}^k$ and $\|u\|_{\infty}=1$. By denoting with $M_{ij}^{(k)}$ and $u_j$ the entries of $M^{(k)}$ and $u$, respectively, we have
\begin{eqnarray*}
\|M^{(k)} u\|_1&=&\Sigma_{i=1}^k \ \left|\Sigma_{j=1}^k \ M_{ij}^{(k)}u_j \right|\\
&\le& \Sigma_{i=1}^k \ \Sigma_{j=1}^k \ \left|M_{ij}^{(k)}u_j \right| = \Sigma_{i,j=1}^k \ \left|M_{ij}^{(k)}\right|\left|u_j\right|  \\
&\le& \Sigma_{i,j=1}^k \ \left|M_{ij}^{(k)}\right|=\|M^{(k)}\|_1, 
\end{eqnarray*}
and this shows that 
$$
 \|M^{(k)}\|_1 \ge \|M^{(k)}\|_{\infty,1}.
$$ 
Now, if $I_k$ is the identity matrix $k \times k$, by resorting to $u=\left[\begin{matrix} 1 & \dots & 1\end{matrix}\right]^T$
one obtains $\|I_k\|_{\infty,1}= \|I_k\|_1=k$ thus proving the existence of $M_1^{(k)}=I_k$. 
Now, let's define
$$
f:{\mathcal C}_k / \{ \ 0 \ \} \ \rightarrow \ {\mathbb R}, \ f(M_{(k)}):=\frac{\|M^{(k)}\|_{\infty,1}}{\|M^{(k)}\|_1}.
$$
Such function is continuous since $\|M^{(k)}\|_1=0$ if and only if  $M^{(k)}=0$ and because  $\|M^{(k)}\|_{\infty,1}$ as well as $\|M^{(k)}\|_1$ 
are continuous maps of the $M^{(k)}$'s entries (thanks also to Lemma \ref{Lemma1} that clearly holds true even for square matrices). Since $f(\alpha M^{(k)})=f(M^{(k)})$ for any $\alpha > 0$, to assess the values that such function can assume it suffices to consider the matrices satisfying $\|M^{(k)}\|_1=1$. The corresponding subset ${\mathcal S}_k$ of ${\mathcal C}_k$ is a compact set\footnote{Denoting by $m_{ij}$ the entries of $M^{(k)}$, we have a set of equalities/inequalities which define the structure of ${\mathcal S}_k$:
\begin{itemize}
\item $m_{ij}=m_{ji}$ for any $i,j$ (due to the symmetry constraint);
\item $\Sigma_{ij} \ |m_{ij}|=1$ (due to the unit $\ell_1$ norm constraint);
\item various polynomial inequalities of the (closed) form $p_h(m_{ij}) \ge 0$ (due to the set of Sylvester's inequalities).
\end{itemize}
These set of conditions makes ${\mathcal S}_k$ bounded - because of the second equality - and closed - as a consequence of the equality/(closed) inequalities. Compactness is therefore guaranteed.}, hence $f$ admits both a minimum and a maximum. The maximum corresponds to 1 (since, as already seen, $\|M^{(k)}\|_{\infty,1} \le \|M^{(k)}\|_1$ and thanks to the existence of $M_1^{(k)}$), the minimum is non-negative and not larger than 1. Consequently, at least a matrix $M_2^{(k)}$ exists that defines the minimum value, i.e. $\lambda(k)$. Finally, since the block diagonal matrix $\diag(0,M^{(k)})$ (with $0$ of size $1 \times 1$) belongs to ${\mathcal C}_{k+1}$ for any $M^{(k)} \in {\mathcal C}_k$, and since the two norms for $\diag(0,M^{(k)})$ coincide with those of $M^{(k)}$, the last property $\lambda(k+1) \le \lambda(k)$ is immediately obtained.
\end{proof}

\medskip

The sequence $\lambda(k)$ plays a central role for our analysis.
In fact, it is now shown that the asymptotic behavior of $\lambda(k)$ uniquely determines whether absolute summability is or not a necessary and sufficient condition for a symmetric positive semidefinite operator to map all the space $\ell_{\infty}$ into $\ell_1$. 
This fact represents the second main result of this paper and
is contained in the next theorem. When reading it, recall from \cite{ChenStableRKHS}[Lemma 4.1]
that if an integral operator maps the entire $\ell_{\infty}$ into $\ell_1$ then it must be bounded (this point is further discussed in Remark  \ref{RemBell}).

 
\medskip

\begin{theorem} \label{lambda} Let $\lambda_{\infty}:=\lim_{k \rightarrow +\infty} \ \lambda(k)$. Then, $\lambda_{\infty}>0$ implies that absolute summability is a necessary and sufficient condition for a symmetric positive semidefinite operator ${\mathcal M}$ from $\ell_{\infty}$ into $\ell_1$ to be bounded. 
Instead, $\lambda_{\infty}=0$ implies that there exist bounded symmetric positive semidefinite operators ${\mathcal M}$ from $\ell_{\infty}$ into $\ell_1$  that are not absolutely summable.  
\end{theorem}

\begin{proof} Since $\lambda(k)$ is monotone non-increasing and lower bounded by 0, $\lambda_{\infty} \ge 0$ exists.
 Assume that $\lambda_{\infty}>0$ and $\|{\mathcal M}\|_1=+\infty$.
 Let also $N>0$ be a fixed real number and denote with $Q_k$
 the sequence of finite submatrices $Q_k$, of size $k \times k$,
 built with the first $k$ rows and columns of the infinite matrix which defines ${\mathcal M}$.
 Since $\|{\mathcal M}\|_1=+\infty$, the 
 $\|Q_k\|_1$ represent a monotone non-decreasing sequence and
 one has
 $$
 \lim_{k \rightarrow +\infty} \ \|Q_k\|_1=+\infty.
 $$ 
So, there exists $k(N)>0$ such that $\|Q_r\|_1 \ge \frac{N}{\lambda_{\infty}}$ for any $r \ge k(N)$, and this implies
$$
\|Q_r\|_{\infty,1} \ge \lambda( r )\|Q_r\|_1 \ge \lambda_{\infty} \|Q_r\|_1 \ge N, \ \forall \ r \ge k(N).
$$
Since $N>0$ was arbitrary, one also has 
$$
\lim_{k \rightarrow +\infty} \ \|Q_k\|_{\infty,1}=+\infty
$$
which clearly prevents ${\mathcal M}$ to be a bounded operator. 
On the other hand, the condition $\|{\mathcal M}\|_1<+\infty$ clearly implies that ${\mathcal M}$ is bounded.
So, we have proved that absolute summability is the necessary and sufficient condition for the operator ${\mathcal M}$ to be bounded
if $\lambda_{\infty}>0$.\\
Assume now that $\lambda_{\infty}=0$. In this case, there exists a sequence $n(k)$ such that $\lambda(n(k)) \le \frac{1}{k}$. 
According to Lemma \ref{Lemma5}, consider matrices $Q_{n(k)} \ne 0$  corresponding to $n(k)$ and such that
$$
\|Q_{n(k)}\|_{\infty,1}=\lambda_{n(k)}\|Q_{n(k)}\|_1 \le \frac{1}{k}\|Q_{n(k)}\|_1.
$$
Similarly to what done in Section \ref{Sec2} let us normalize $Q_{n(k)}$ in such a way that its $\ell_1$ norm becomes $\frac{1}{k}$,
i.e. we define 
$$
S_k:=\frac{Q_{n(k)}}{k\|Q_{n(k)}\|_1}.
$$
so that
$$
\|S_k\|_1=\frac{1}{k} 
$$
and
$$
\|S_k\|_{\infty,1} \le \frac{1}{k} \|S_k\|_1 \le \frac{1}{k^2}.
$$
Now, the desired counterexample is found by choosing the infinite matrix $M$ that defines ${\mathcal M}$ as follows 
$$
M=\diag(S_1,S_2,\dots).
$$
In fact, the equalities 
$$
\|{\mathcal M}\|_1=\Sigma_{k=1}^{+\infty} \ \frac{1}{k} = +\infty
$$
and
$$
\|{\mathcal M}\| \le \Sigma_{k=1}^{+\infty} \ \frac{1}{k^2} < +\infty
$$
show that ${\mathcal M}$ is a bounded operator which is not absolutely summable.
\end{proof}
\medskip

\begin{remark} \label{RemBell}The previous theorem gives the necessary and sufficient condition for the (possible) existence of symmetric positive semidefinite bounded operators which are not absolutely summable. On the other hand, RKHS stability is related to operators which are well-defined over the whole $\ell_{\infty}$ (${\mathcal M}u$ has to belong to $\ell_1$ for any $u \in \ell_{\infty}$), a property that would seem to be different from boundedness. But, as already recalled in introducing 
Theorem \ref{lambda}, while boundedness obviously implies well-definiteness (so, if $\lambda_{\infty}=0$ there is nothing else to prove), the converse also holds true as a consequence of 
 \cite{ChenStableRKHS}[Lemma 4.1]. This fact would be fundamental in the case $\lambda_{\infty}>0$
to show that absolute summability is equivalent to RKHS stability. However, we will prove in the next section that $\lambda_{\infty}=0$ and this makes the outcomes in  \cite{ChenStableRKHS} 
irrelevant for our developments. 
\end{remark}


\section{A class of important positive semidefinite matrices}\label{Sec6}

In this section we analyze properties of some key symmetric and positive semidefinite matrices
that will lead to the building of the second counterexample (and, hence, to the solution of our main problem).
Remarkably, such matrices are 
defined in terms of the matrices $V^{(n)}$ already encountered in the previous sections.
They are in fact given by 
$$
M^{(n)}=V^{(n)} V^{(n)\top}.
$$

\begin{lemma}\label{Lemma6} The columns of the matrix $V^{(n)}$ are orthogonal each other and one has
$$
V^{(n)\top}  V^{(n)} = n I_m, 
$$
where $I_m$ is the identity matrix of size $m$.
\end{lemma}

\begin{proof} Fix two distinct integers $i$ and $j$ less than or equal to $m$.
For $k=1,\ldots,n$, the couples $(V_{ki}^{(n)},V_{kj}^{(n)})$ may be $(1,1),(-1,-1),(-1,1)$ or $(1,-1)$.
The number of couples of the first type $(1,1)$ are $2^{m-2}$ since,
by construction, they are complemented with any combination of $m-2$ signs $\pm 1$.
The same holds exactly for the other three couples, hence
$$
(V^{(n)\top} V^{(n)})_{i,j}=\Sigma_{k=1}^n \ V_{ki}^{(n)}V_{kj}^{(n)}  = 2^{m-2} + 2^{m-2} -2^{m-2} -2^{m-2} =0
$$
If $i=j$, one instead has $(V^{(n)\top} V^{(n)})_{i,i}= \|c_i\|_2^2=n$ (with $c_i$ a column of $V^{(n)}$) and this completes the proof. 
\end{proof}

For future developments, it is now important to provide insights regarding
$$
{\mathcal M}_n:=max_{u \in {\mathcal U}_n} \ \|M^{(n)}u\|_1
$$ 
and 
$$
{\mathcal M}_n^*:=n \cdot max_{a \in {\mathbb R}^m, \ \|a\|_2 \le 1} \ \|V^{(n)}a\|_1.
$$
\medskip
\begin{lemma} \label{Lemma7}  ${\mathcal M}_n^* \ge {\mathcal M}_n$ holds true.
\end{lemma}

\begin{proof} We start by decomposing $u \in {\mathcal U}_n$ in terms of orthogonal components, i.e.
$$
u=a_1c_1+\dots+a_mc_m+w=V^{(n)}a+w, 
$$
where recall that the $c_i$ are the columns of $V^{(n)}$, 
$$
\ a:=\left[\begin{matrix} a_1 & \dots & a_m\end{matrix}\right]^T
$$
and
$$
w \perp c_i, \ i=1,\dots,m, \quad w^TV^{(n)}=0.
$$
Then, it follows that, for any $u \in {\mathcal U}_n$, one has
$$
n=\|u\|_2^2=\|a_1c_1\|_2^2+\dots+\|a_mc_m\|_2^2+\|w\|_2^2=n(a_1^2+\dots+a_m^2)+\|w\|_2^2.
$$
Thus, one has $\|a\|_2 \le 1$ and also 
\begin{eqnarray*} 
M^{(n)}u&=&M^{(n)}(V^{(n)}a+w)=V^{(n)}(V^{(n)\top}V^{(n)})a+V^{(n)}V^{(n)\top}w \\
&=&V^{(n)}(nI_m)a+V^{(n)}(w^TV^{(n)})^T=nV^{(n)}a+V^{(n)} \cdot 0 \\
&=& nV^{(n)}a.
\end{eqnarray*}
We obtain
$$
\|M^{(n)}u\|_1=n\|V^{(n)}a\|_1
$$
whose maximum value can be found by inspecting either the finite set ${\mathcal U}_n$ for evaluating ${\mathcal M}_n$, or the ipersphere defined by $\|a\|_2 \le 1$ 
(that contains the $2^n$ points corresponding to the various vectors $u$) for evaluating ${\mathcal M}_n^*$. 
The proof is then completed just noticing that $max_{x \in A} \ f(x) \ \le max_{x \in B} \ f(x)$ if $A \subset B$.
\end{proof}

\medskip
\begin{lemma} \label{Lemma8}  It holds that
$$
{\mathcal M}_n^*=n \cdot max_{\|a\|_2=1} \ \Sigma_{b \in {\mathcal U}_m} \ |a^Tb|
$$
\end{lemma} 
\begin{proof}   From
$$
nV^{(n)}a=n(a_1c_1+\dots+a_mc_m)
$$
and by the properties of $V^{(n)}$'s rows, the entries of the vector $a_1c_1+\dots+a_mc_m$ 
are given by the coefficients $a_i, \ i=1,\dots,m$ multiplied by $\pm 1$ in all the possible $n=2^m$ combinations.
So, if $b:=\left[\begin{matrix}b_1 & \dots & b_m \end{matrix} \right]^T \in {\mathcal U}_m$, we easily have
\begin{eqnarray*} 
n\|V^{(n)}a\|_1&=&n\|a_1c_1+\dots+a_mc_m\|_1\\
&=&n\Sigma_{b \in {\mathcal U}_m} \ |a_1b_1+\dots+a_mb_m|\\
&=& n\Sigma_{b \in {\mathcal U}_m} \ |a^Tb|
\end{eqnarray*} 
where $\|a\|_2 \le 1$. If $\|a\|_2<1$, there exists some $h>1$ such that $\|ha\|_2 \le 1$ and 
$$
n\Sigma_{b \in {\mathcal U}_m} \ |ha^Tb|=hn\Sigma_{b \in {\mathcal U}_m} \ |a^Tb| > n\Sigma_{b \in {\mathcal U}_m} \ |a^Tb|.
$$ 
Then, we conclude that the maximum point must belong to the boundary $\|a\|_2=1$.
\end{proof}
\medskip
\begin{lemma} \label{Lemma9}  It holds that
$$
\Sigma_{b \in {\mathcal U}_m} \ (a_1b_1+\dots+a_mb_m)^2=n\|a\|_2^2
$$
\end{lemma} 
\begin{proof}   By developing the squares, also recalling that $b_i=\pm 1$ implies $b_i^2=1$, one has
$$
\begin{array}{lcl}
\Sigma_{b \in {\mathcal U}_m} \ (a_1b_1+\dots+a_mb_m)^2&=&\Sigma_{b \in {\mathcal U}_m} \ [a_1^2+\dots+a_m^2]+\Sigma_{b \in {\mathcal U}_m} \ (2\Sigma_{i \ne j} \ a_ia_jb_ib_j) \cr
&=&[a_1^2+\dots+a_m^2]2^m+2\Sigma_{i \ne j} a_ia_j (\Sigma_{b \in {\mathcal U}_m} \ b_i b_j) \cr
&=&n\|a\|_2^2+2\Sigma_{i \ne j} a_ia_j (\Sigma_{b \in {\mathcal U}_m} \ b_i b_j)
\end{array}
$$
The conclusion then follows by noticing that $\Sigma_{b \in {\mathcal U}_m} \ b_ib_j=0$. In fact, the $b_i$  assume the values $\pm 1$ in all possible combinations.
So, the pairs $(b_i,b_j)=(1,1),(1,-1),(-1,1),(-1,-1)$ appear the same number of times implying that the terms $b_ib_j=1$ and $b_ib_j=-1$ appear the same number of times, too. 
\end{proof}  
 \medskip
\begin{theorem}\label{theorem2} ${\mathcal M}_n={\mathcal M}_n^*=n^2$ holds true.
\end{theorem}

\begin{proof}  By choosing $u$ equal to any column of $V^{(n)}$, one easily obtains $M^{(n)}u=nu$ and
this implies ${\mathcal M}_n \ge n\|u\|_1=n^2$. From the inequality $(\Sigma C_i^2)(\Sigma D_i^2) \ge (\Sigma C_iD_i)^2$, with $C_i=|a_1b_1+\dots+a_mb_m|, \ D_i=1$ and from Lemma \ref{Lemma9}, we obtain
$$
\begin{array}{lcl}
\left(\Sigma_{b \in {\mathcal U}_m} \ |a^Tb|\right)^2&=&\left[\Sigma_{b \in {\mathcal U}_m} \ |a_1b_1+\dots+a_mb_m| \ \cdot 1\right]^2 \cr
&\le& [\Sigma_{b \in {\mathcal U}_m} \ (a_1b_1+\dots+a_mb_m)^2] \ [\Sigma_{b \in {\mathcal U}_m} \ 1]=n\|a\|_2^22^m=n^2\|a\|_2^2
\end{array}
$$
and hence also
$$
\Sigma_{b \in {\mathcal U}_m} \ |a^Tb| \le n\|a\|_2.
$$ 
By exploiting Lemmas \ref{Lemma7} and \ref{Lemma8}, 
one obtains ${\mathcal M}_n \le {\mathcal M}_n^* \le n^2$ and, since we proved that ${\mathcal M}_n \ge n^2$, 
the conclusion is obtained. 
\end{proof}  

\section{Kernel absolute summability is only sufficient for RKHS stability}\label{Sec7}
We can now prove that $\lambda_{\infty}=0$ by exploiting the properties of the matrices $M^{(n)}$ and the previously obtained results.
This will allow to build the second counterexample that shows that bounded operators from $\ell_{\infty}$ into $\ell_1$ exist 
in absence of absolute summability even when the symmetric positive semidefinite constraints (\ref{Kconstraint}) are active. 
\medskip
\begin{lemma} \label{Lemma10} For large $p,m,n$, it holds that
$$
\frac{\|M^{(n)}\|_{\infty,1}}{\|M^{(n)}\|_1} \simeq \sqrt{\frac{\pi}{4p}}.
$$
Hence, it holds that
$$
\lambda_{\infty}:=\lim_{k \rightarrow +\infty} \ \lambda(k) = 0.
$$
\end{lemma}
\begin{proof}  Let $e_i, \ i=1,2,\dots,n$ be the canonical basis in ${\mathbb R}^n$.
 Exploiting Lemma \ref{Lemma2}, one has that the $i-$th column of $M^{(n)}$ satisfies 
 $$
\|M^{(n)}e_i\|_1=\|V^{(n)}V^{(n)\top}e_i\|_1=\|V^{(n)}v^{(i)}\|_1 = \|V^{(n)}\|_{\infty,1}
$$
where the last equality exploits both the fact 
that $v^{(i)} \in {\mathcal U}_m$ for any $i$ and Lemma \ref{Lemma3}.
Hence, 
\begin{eqnarray*}
\|M^{(n)}\|_1&=&\sum_{h=1}^n \ \|M^{(n)}e_i\|_1=\sum_{h=1}^n \ \|V^{(n)}v^{(i)}\|_1 \\
&=&n\|V^{(n)}v^{(1)}\|_1=n\|V^{(n)}\|_{\infty,1}
\end{eqnarray*}
and from ({\ref A}), for  $p,m,n$ large enough it holds that
$$
\|M^{(n)}\|_1 \simeq 2n^2\sqrt{\frac{p}{\pi}}.
$$
In addition, from Theorem \ref{theorem2} and Lemma \ref{Lemma1} (now applied to $M^{(n)}$) one has
$$
\|M^{(n)}\|_{\infty,1}=n^2
$$
so that
$$
\frac{\|M^{(n)}\|_{\infty,1}}{\|M^{(n)}\|_1} \simeq \sqrt{\frac{\pi}{4p}}.
$$
In view of the last result, now it is easy to see that $\lambda_{\infty}=0$. In fact, 
we have just considered special matrices $M^{(n)} \in {\mathcal C}_n$, so that for $p$ large enough
\begin{eqnarray}\label{TheBound}
\lambda(n) &=& \lambda(2^{2p+1}) \le \frac{\|M^{(2^{2p+1})}\|_{\infty,1}}{\|M^{(2^{2p+1})}\|_1} \simeq \sqrt{\frac{\pi}{4p}} =  \frac{\sqrt{\pi}}{\sqrt{2\log_2(n(p))-2}}.
\end{eqnarray}
The upper bound established by (\ref{TheBound}) holds for the special values $n=2^{2p+1}$.
 But, 
since $\lambda(k)$ is  a monotone non-increasing sequence, this indeed implies $\lambda_{\infty}=0$
(as also graphycally depicted in Fig. \ref{Fig1}).
\end{proof}

\medskip
The result $\lambda_{\infty}=0$ just obtained  
paves the way for the most important result of the paper.
It is achieved through the following counterexample that shows that 
kernel absolute summability is only sufficient for RKHS stability.


\begin{figure}
  \begin{center}
   \begin{tabular}{cc}
\hspace{-.2in}
 { \includegraphics[scale=0.35]{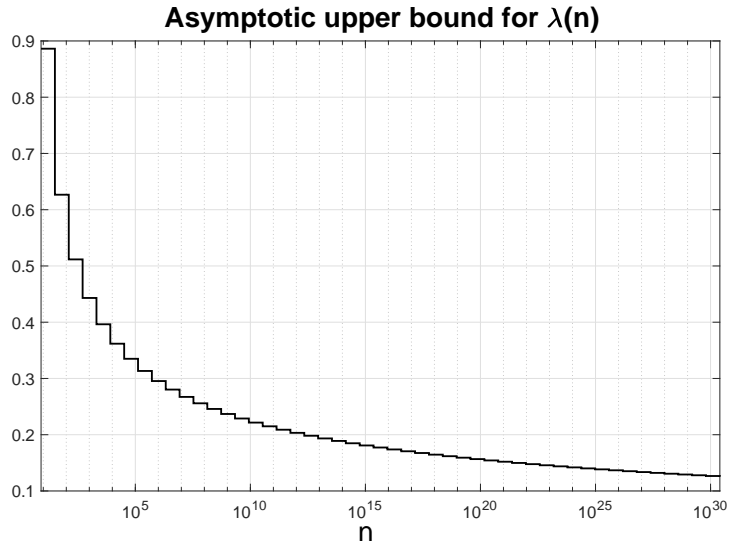}}
    \end{tabular}
    \caption{Let $n(p)=2^{2p+1}$ with  
    $p=1,2,\ldots$. The displayed curve is piecewise constant:
    over the intervals $n(p) \leq n < n(p+1)$ it is  
    equal to $\frac{\sqrt{\pi}}{\sqrt{2\log_2(n(p))-2}}$. Then, Lemma \ref{Lemma10}, as synthesized also by (\ref{TheBound}), ensures that
    this curve becomes asymptotically an upper bound for $\lambda(n)$. So, $\lambda(n)$ converges to zero
    with a rate not slower than the inverse of $\sqrt{\log_2(n)}$.}
    \label{Fig1}
     \end{center}
\end{figure}

\medskip

\begin{counterexample} A linear symmetric positive semidefinite operator ${\mathcal S}: \ \ell_{\infty} \ \rightarrow \ \ell_1$ can be bounded despite the lack of absolute summability.
\end{counterexample}

\begin{proof} We just need to exploit the result $\lambda_{\infty}=0$ obtained above. In particular, let us normalize the matrices $M^{(n)}$ as follows: we find a sequence $p(h), \ h=1,2,\dots$ such that $\lambda(2^{2p(h)+1}) \le \frac{1}{h}$\footnote{A simple choice is $p(h)=h^2$. Even if it guarantees that the inequality about $\lambda(2^{2p(h)+1})$ is satisfied only for  large $h$, this is actually all is needed since the convergence of a series only depends on the asymptotic behavior of its terms. An even simpler
choice could be $p(h)=h$: recalling the arguments in the proof of Theorem \ref{lambda} 
this would imply boundedness of the operator ${\mathcal S}$ as a consequence of the convergence of the series $\sum_{p=1}^{+\infty} \ \frac{1}{p\sqrt{p}}$.}, and then define
$$
S_h:=\frac{1}{h}\frac{M^{(2^{2p(h)+1})}}{\|M^{(2^{2p(h)+1})}\|_1}, \ \ h=1,2,\dots
$$
that implies both $\|S_h\|_1=\frac{1}{h}$ and $\|S_h\|_{\infty,1} \le \frac{1}{h^2}$. The infinite matrix $S=\diag(S_1,S_2,\dots)$ leads to an operator ${\mathcal S}:\ell_{\infty} \rightarrow \ell_1$, associated with $S$, which is bounded but it is not absolutely summable. Moreover, it is easily seen that the equality $\|{\mathcal S}u\|_1=\|{\mathcal S}\|=\|{\mathcal S}\| \cdot \|u\|_{\infty}$  
holds for any $u \in {\mathcal U}_{\infty}$. 
\end{proof}

\section{Conclusions}
Many authors pointed out that kernel absolute summability is a sufficient condition for RKHSs stability, without elaborating on its possible necessity. 
 None of the two possibilities was supported as the most reasonable one:
no clues were available, and the secret hope for a (surely desired) equivalence was postponed to further investigations.
Now we can claim (unfortunately, in some sense) that the class of stable RKHSs is wider than that of absolutely summable kernels. As we have described, the idea behind the counterexample construction is somewhat involved. This shows that such (no longer) open problem 
was an hard task to deal with, requiring understanding of the real nature of two different norms connected with operators mapping $\ell_{\infty}$ into $\ell_1$.

\bibliographystyle{abbrv}
\bibliography{references_sasha,biblio}

\begin{thebibliography}{10}

\bibitem{Akaike1974}
H.~Akaike.
\newblock A new look at the statistical model identification.
\newblock {\em IEEE Transactions on Automatic Control}, 19:716--723, 1974.

\bibitem{AravkinNN}
A.~Aravkin, B.~Bell, J.~Burke, and G.~Pillonetto.
\newblock The connection between {B}ayesian estimation of a {G}aussian random
  field and {RKHS}.
\newblock {\em IEEE Trans. on Neural Networks and Learning Systems},
  26(7):1518--1524, 2015.

\bibitem{Carmeli}
C.~Carmeli, E.~D. Vito, and A.~Toigo.
\newblock {V}ector valued reproducing kernel {H}ilbert spaces of integrable
  functions and {M}ercer theorem.
\newblock {\em Analysis and Applications}, 4:377--408, 2006.

\bibitem{ChenKS2018}
T.~Chen.
\newblock On kernel design for regularized lti system identification.
\newblock {\em Automatica}, 90:109 -- 122, 2018.

\bibitem{ChenKS2015}
T.~Chen and L.~Ljung.
\newblock On kernel structures for regularized system identification (ii): a
  system theory perspective.
\newblock {\em IFAC-PapersOnLine}, 48(28):1041 -- 1046, 2015.
\newblock 17th IFAC Symposium on System Identification SYSID 2015.

\bibitem{ChenStableRKHS}
T.~Chen and G.~Pillonetto.
\newblock On the stability of reproducing kernel hilbert spaces of
  discrete-time impulse responses.
\newblock {\em Automatica}, 95:529 -- 533, 2018.

\bibitem{Darwish2015}
M.~{Darwish}, G.~{Pillonetto}, and R.~{Tóth}.
\newblock Perspectives of orthonormal basis functions based kernels in bayesian
  system identification.
\newblock In {\em 2015 54th IEEE Conference on Decision and Control (CDC)},
  pages 2713--2718, 2015.

\bibitem{Desoer1975}
C.~Desoer and M.~Vidyasagar.
\newblock {\em Feedback systems: input-output properties}.
\newblock Academic Press, 1975.

\bibitem{DinuzzoSIAM15}
F.~Dinuzzo.
\newblock Kernels for linear time invariant system identification.
\newblock {\em SIAM Journal on Control and Optimization}, 53(5):3299--3317,
  2015.

\bibitem{Drucker97}
H.~Drucker, C.~Burges, L.~Kaufman, A.~Smola, and V.~Vapnik.
\newblock Support vector regression machines.
\newblock In {\em Advances in Neural Information Processing Systems}, 1997.

\bibitem{PoggioMIT}
T.~Evgeniou, M.~Pontil, and T.~Poggio.
\newblock Regularization networks and support vector machines.
\newblock {\em Advances in Computational Mathematics}, 13:1--150, 2000.

\bibitem{Fujimoto2017}
Y.~Fujimoto, I.~Maruta, and T.~Sugie.
\newblock Extension of first-order stable spline kernel to encode relative
  degree.
\newblock {\em IFAC-PapersOnLine}, 50(1):14016 -- 14021, 2017.
\newblock 20th IFAC World Congress.

\bibitem{Girosi:1997}
F.~Girosi.
\newblock An equivalence between sparse approximation and support vector
  machines.
\newblock Technical report, Cambridge, MA, USA, 1997.

\bibitem{Hastie01}
T.~Hastie, R.~Tibshirani, and J.~Friedman.
\newblock {\em The Elements of Statistical Learning. Data Mining, Inference and
  Prediction}.
\newblock Springer, Canada, 2001.

\bibitem{Kimeldorf71Bayes}
G.~Kimeldorf and G.~Wahba.
\newblock A correspondence between {B}ayesan estimation of stochastic processes
  and smoothing by splines.
\newblock {\em Ann. Math. Statist.}, 41(2):495--502, 1971.

\bibitem{Ljung}
L.~Ljung.
\newblock {\em System Identification, Theory for the User}.
\newblock Prentice Hall, 1999.

\bibitem{Lukic}
M.~Lukic and J.~Beder.
\newblock Stochastic processes with sample paths in reproducing kernel
  {H}ilbert spaces.
\newblock {\em Trans. Amer. Math. Soc.}, 353:3945--3969, 2001.

\bibitem{SS2010}
G.~Pillonetto and G.~{De Nicolao}.
\newblock A new kernel-based approach for linear system identification.
\newblock {\em Automatica}, 46(1):81--93, 2010.

\bibitem{SurveyKBsysid}
G.~Pillonetto, F.~Dinuzzo, T.~Chen, G.~D. Nicolao, and L.~Ljung.
\newblock Kernel methods in system identification, machine learning and
  function estimation: a survey.
\newblock {\em Automatica}, 50(3):657--682, 2014.

\bibitem{Poggio90}
T.~Poggio and F.~Girosi.
\newblock {N}etworks for approximation and learning.
\newblock In {\em Proceedings of the {IEEE}}, volume~78, pages 1481--1497,
  1990.

\bibitem{Rasmussen}
C.~Rasmussen and C.~Williams.
\newblock {\em {G}aussian Processes for Machine Learning}.
\newblock The MIT Press, 2006.

\bibitem{RTR}
R.~Rockafellar.
\newblock {\em {Convex Analysis}}.
\newblock Princeton Landmarks in Mathematics. Princeton University Press, 1970.

\bibitem{Scholkopf01b}
B.~Sch\"{o}lkopf and A.~J. Smola.
\newblock {\em Learning with Kernels: Support Vector Machines, Regularization,
  Optimization, and Beyond}.
\newblock (Adaptive Computation and Machine Learning). MIT Press, 2001.

\bibitem{schwarz1978estimating}
G.~Schwarz et~al.
\newblock Estimating the dimension of a model.
\newblock {\em The annals of statistics}, 6(2):461--464, 1978.

\bibitem{Soderstrom}
T.~S{\"o}derstr{\"o}m and P.~Stoica.
\newblock {\em System Identification}.
\newblock Prentice-Hall, 1989.

\bibitem{Vapnik98}
V.~Vapnik.
\newblock {\em Statistical Learning Theory}.
\newblock Wiley, New York, NY, USA, 1998.

\bibitem{Wahba1990}
G.~Wahba.
\newblock {\em {Spline Models For Observational Data}}.
\newblock SIAM, Philadelphia, 1990.

\bibitem{Willems1970}
J.~Willems.
\newblock {\em Stability theory of dynamical systems}.
\newblock Wiley, 1970.

\end{thebibliography}

\end{document}